\documentclass{llncs}
\usepackage{amsmath}
\usepackage{amssymb}
\usepackage{times}
\usepackage{indentfirst}
\usepackage{graphicx}

\begin{document}

\title{Parametric matroid of rough set}

\author{Yanfang Liu, William Zhu~\thanks{Corresponding author.
E-mail: williamfengzhu@gmail.com (William Zhu)}}
\institute{
Lab of Granular Computing\\
Zhangzhou Normal University, Zhangzhou 363000, China}



\date{\today}          
\maketitle

\begin{abstract}
Rough set is mainly concerned with the approximations of objects through an equivalence relation on a universe.
Matroid is a combinatorial generalization of linear independence in vector spaces.
In this paper, we define a parametric set family, with any subset of a universe as its parameter, to connect rough sets and matroids.
On the one hand, for a universe and an equivalence relation on the universe, a parametric set family is defined through the lower approximation operator.
This parametric set family is proved to satisfy the independent set axiom of matroids, therefore it can generate a matroid, called a parametric matroid of the rough set.
Three equivalent representations of the parametric set family are obtained.
Moreover, the parametric matroid of the rough set is proved to be the direct sum of a partition-circuit matroid and a free matroid.
On the other hand, since partition-circuit matroids were well studied through the lower approximation number, we use it to investigate the parametric matroid of the rough set.
Several characteristics of  the parametric matroid of the rough set, such as independent sets, bases, circuits, the rank function and the closure operator, are expressed by the lower approximation number.

\end{abstract}


\section{Introduction}
Rough set theory is based on equivalence relations, and it was proposed by Pawlak to handle incomplete and inexact knowledge in information systems.
It is an extension of set theory for studying and analyzing various types of data~\cite{Pawlak82Rough,Pawlak91Rough}.
Rough set theory has been successfully applied to many fields, such as machine learning~\cite{HuPanAnMaWei11Anefficient,MinCaiLiuBai07Dynamic}, pattern recognition~\cite{Li04Topological,RaiffaRichardsonMetcallfe02Negotiation}, intelligent decision making~\cite{Yao10Three-way}, granular computing~\cite{CalegariCiucci10Granularcomupting,ZhuWang06Covering}, data mining~\cite{LinCercone97Roughsets,PolkowskiSkowron98Rough2}, approximate reasoning~\cite{BittnerStell00Rough,Zadeh75III}, attribute reduction~\cite{HallHolmes03Benchmarking,MiaoZhaoYaoLiXu09relativereducts,MinHeQianZhu11Test,QianLiangPedryczDang10Positive,ZhangQiuWu07ageneralapproach}, rule induction~\cite{Kryszkiewicz98Rule,YaoWangWang05Rule} and others~\cite{ChenMiaoWang10aroughsetapproach,DaiXu12approximations,WeiLiangQian12acomparative}.
Moreover, through extending equivalence relations or partitions, some extensions of rough sets are proposed, such as generalized rough sets base on relations~\cite{Kryszkiewicz98Rough,LiuWang06Semantic,QinYangPei08generalizedroughsets,SlowinskiVanderpooten00AGeneralized,Yao98Constructive,Yao98Relational,ZhuWang06Binary}, and covering-based rough sets~\cite{BonikowskiBryniarskiWybraniecSkardowska98Extensions,Zhu07Topological,ZhuWang03Reduction,ZhuWang06Axiomatic,ZhuWang06ANew,ZhuWang06Relationships}.

Matroid theory~\cite{Lai01Matroid,Mao06TheRelation} was proposed by Whitney to generalize the essence of ``independence'' in linear algebra.
Matroids have sound theoretical foundations and wide applications.
In theory, matroids have powerful axiomatic systems which provide a platform for connecting them with other theories, such as rough sets~\cite{LiuZhuZhang12Relationshipbetween,LiuZhu12characteristicofpartition-circuitmatroid}, generalized rough sets based on relations~\cite{WangZhuMin11TheVectorially,ZhangWangFengFeng11reductionofrough,ZhuWang11Matroidal,ZhuWang11Rough}, covering-based rough sets~\cite{WangZhuMin11Transversal,WangZhu11Matroidal} and geometric lattices~\cite{AignerDowling71matchingtheory,Matus91abstractfunctional}.
In application, matroids have been used in diverse fields, such as combinatorial optimization~\cite{Lawler01Combinatorialoptimization}, algorithm design~\cite{Edmonds71Matroids},  information coding~\cite{RouayhebSprintsonGeorghiades10ontheindex} and cryptology~\cite{DoughertyFreilingZeger07Networks}.

In this paper, for a universe and an equivalence relation on the universe, we define a parametric set family, with any subset of the universe as its parameter, is defined to connect rough sets and matroids.
Firstly, for any subset, the parametric set family is proved to satisfy the independent set axiom of matroids, then a matroid called a parametric matroid of the rough set with respect to the subset is generated by the parametric set family.
Two equivalent representations of the parametric set family are obtained through the lower approximation operator, and another equivalent representation is expressed by the partition generated by the equivalence relation.
Moreover, the parametric matroid of the rough set with respect to the subset is proved to be the direct sum of a partition-circuit matroid  and a free matroid, where the partition-circuit matroid is based on the restriction of the equivalence relation in the complement of the lower approximation of the subset and the free matroid is based on the lower approximation of the subset.
The partition-circuit matroid is the restriction of the parametric matroid of the rough set, and so is the free matroid.
Secondly, several characteristics of the parametric matroid of the rough set are studied by the lower approximation number which is proposed in~\cite{LiuZhu12characteristicofpartition-circuitmatroid}.
Since a partition-circuit matroid was well investigated through the lower approximation number in~\cite{LiuZhu12characteristicofpartition-circuitmatroid}, we use it to study the parametric matroid of the rough set as the direct sum of a partition-circuit matroid and a free matroid.
Independent sets, bases, circuits, the rank function and the closure operator of the parametric matroid of the rough set are well expressed by the lower approximation number.

The rest of this paper is organized as follows: In Section~\ref{S:preliminaries}, we recall some basic definitions of classical rough sets and matroids.
Section~\ref{S:aparametricmatroidofroughsets} defines a parametric set family and proves it to be an independent set family of a matroid which is called a parametric matroid of rough sets.
In Section~\ref{S:characteristicsoftheparametricmatroid}, we study characteristics of the parametric matroid of the rough set through the lower approximation number.
Finally, we conclude this paper in Section~\ref{S:conclusions}.

\section{Preliminaries}
\label{S:preliminaries}
In this section, we recall some basic definitions and related results which will be used in this paper.

\subsection{Binary relation}

Let $U$ be a universe.
If $R\in U\times U$, then $R$ is called a binary relation~\cite{RajagopalMason92Discrete} on $U$.
For all $(x, y)\in U\times U$, if $(x, y)\in R$, we say $x$ has relation $R$ with $y$, and denote this relationship as $xRy$.

Throughout this paper, a binary relation is simply called a relation.
In the following definition, we will introduce the restriction of a relation.

\begin{definition}(Restriction of a relation~\cite{RajagopalMason92Discrete})
\label{D:restrictionofrelation}
Let $R$ be a relation on $U$ and $X\subseteq U$.
The restriction of $R$ in $X$ is defined as follows:\\
\centerline{$R\upharpoonright X=\{(x, y)\in R:x\in X\}$.}
\end{definition}

We list an example to illustrate the restriction of a relation.

\begin{example}
Let $U=\{1, 2, 3, 4, 5\}$ be a universe, $R=\{(1, 2), (1, 3), (1, 5), (2, 3),$ $ (3, 1), (3, 3), (4, 5), (5, 2)\}$ be a relation on $U$ and $X=\{3, 5\}$. The restriction of $R$ in $X$ is that: $R\upharpoonright X=\{(3, 1), (3, 3), (5, 2)\}$.
\end{example}

Reflective, symmetric, and transitive properties play important roles in characterizing relations.
Then, we introduce equivalence relations through these three properties.

\begin{definition}(Reflexive, symmetric and transitive~\cite{RajagopalMason92Discrete})
Let $R$ be a relation on $U$.\\
If for all $x\in U$, $xRx$, we say $R$ is reflexive.\\
If for all $x, y\in U$, $xRy$ implies $yRx$, we say $R$ is symmetric.\\
If for all $x, y, z\in U$, $xRy$ and $yRz$ imply $xRz$, we say $R$ is transitive.
\end{definition}

\begin{definition}(Equivalence relation~\cite{RajagopalMason92Discrete})
\label{D:equivalenceraltion}
Let $R$ be a relation on $U$.
If $R$ is reflexive, symmetric and transitive, we say $R$ is an equivalence relation on $U$.
\end{definition}

The power of an equivalence relation lies in its ability to partition a set into the disjoint union of subsets called equivalence classes.

\begin{definition}(Equivalence class~\cite{RajagopalMason92Discrete})
Let $R$ be an equivalence on $U$.
For all $x\in U$, $[x]_{R}=\{y\in U:xRy\}$ is called the equivalence class of $x$ with respect to $R$.
\end{definition}

\subsection{Rough set model}

In this subsection, we introduce some concepts and properties of rough sets~\cite{Pawlak82Rough}.

Let $U$ be a non-empty finite set called a universe and $R$ an equivalence relation on $U$.
$R$ will generate a partition $U/R=\{P_{1}, P_{2}, \cdots, P_{m}\}$ on $U$, where $P_{1}, P_{2}, \cdots, P_{m}$ are the equivalence classes, and, in rough sets, they are also called elementary sets of $R$.
For any $X\subseteq U$, we can describe $X$ in terms of the elementary sets of $R$.
Specially, Pawlak~\cite{Pawlak82Rough} introduced two sets called lower and upper approximations.

\begin{definition}(Lower and upper approximations\cite{Pawlak82Rough})
\label{D:lowerandupper}
Let $U$ be a universe and $R$ an equivalence relation on $U$.
For all $X\subseteq U$,\\
\centerline{$\underline{R}(X)=\{x\in U:[x]_{R}\subseteq X\}$}\\
\centerline{$~~~~~~~~~~~~~~~=\cup\{P\in U/R : P\subseteq X\}$,}\\
\centerline{~~~~~~$\overline{R}(X)=\{x\in U:[x]_{R}\cap X\neq\emptyset\}$}\\
\centerline{$~~~~~~~~~~~~~~~~~~~~~=\cup\{P\in U/R : P\cap X\neq\emptyset\}$.}
where $\underline{R}(X), \overline{R}(X)$ is called the lower and upper approximations of $X$ with respect to $R$, respectively.
\end{definition}

In the following proposition, we list only some properties of the lower and upper approximations used in this paper.

\begin{proposition}(\cite{Pawlak82Rough})
\label{P:propertiesofapproximations}
Let $U$ be a universe and $R$ an equivalence relation on $U$.
For all $X, Y\subseteq U$,\\
(1) $\underline{R}(\emptyset)=\emptyset$;\\
(2) $\underline{R}(U)=U$;\\
(3) $\underline{R}(X)\subseteq X$;\\
(4) $\underline{R}(X\cap Y)=\underline{R}(X)\cap\underline{R}(Y)$;\\
(5) $X\subseteq Y\Rightarrow\underline{R}(X)\subseteq\underline{R}(Y)$;\\
(6) $\underline{R}(X)\cup\underline{R}(Y)\subseteq\underline{R}(X\cup Y)$;\\
(7) $\underline{R}(\underline{R}(X))=\underline{R}(X)$;\\
(8) $\overline{R}(X)=U-\underline{R}(U-X)$;\\
(9) $\underline{R}(X)=\overline{R}(\underline{R}(X))$.
\end{proposition}

\subsection{Matroid model}

Matroids have many equivalent definitions.
In the following definition, we will introduce one that focuses on independent sets.

\begin{definition}(Matroid~\cite{Lai01Matroid})
\label{D:matroid}
A matroid is a pair $M=(U, \mathbf{I})$ consisting a finite universe $U$ and a collection $\mathbf{I}$ of subsets of $U$ called independent sets satisfying the following three properties:\\
(I1) $\emptyset\in\mathbf{I}$;\\
(I2) If $I\in \mathbf{I}$ and $I'\subseteq I$, then $I'\in \mathbf{I}$;\\
(I3) If $I_{1}, I_{2}\in \mathbf{I}$ and $|I_{1}|<|I_{2}|$, then there exists $u\in I_{2}-I_{1}$ such that $I_{1}\cup \{u\}\in \mathbf{I}$, where $|I|$ denotes the cardinality of $I$.
\end{definition}

Since the above definition of matroids is defined from the viewpoint of independent sets, it is also called the independent set axiom of matroids.
In order to make some expressions brief, we introduce some symbols as follows.

\begin{definition}(\cite{Lai01Matroid})
Let $U$ be a finite universe and $\mathbf{A}$ a family of subsets of $U$.
Then\\
$Max(\mathbf{A})=\{X\in\mathbf{A}:\forall Y\in\mathbf{A}, X\subseteq Y\Rightarrow X=Y\}$;\\
$Min(\mathbf{A})=\{X\in\mathbf{A}:\forall Y\in\mathbf{A}, Y\subseteq X\Rightarrow X=Y\}$.
\end{definition}

Any maximal independent set of a matroid is a base.
A matroid and its family of bases are uniquely determined by each other.

\begin{definition}(Base~\cite{Lai01Matroid})
\label{D:base}
Let $M=(U, \mathbf{I})$ be a matroid.
Any maximal independent set in $M$ is called a base of $M$, and the family of bases of $M$ is denoted by $\mathbf{B}(M)$, i.e., $\mathbf{B}(M)=Max(\mathbf{I})$.
\end{definition}

In a matroid, a subset is a dependent set if it is not an independent set.
Any circuit of a matroid is a minimal dependent set.
A matroid uniquely determines its circuits, and vice versa.

\begin{definition}(Circuit~\cite{Lai01Matroid})
\label{D:circuit}
Let $M=(U, \mathbf{I})$ be a matroid.
Any minimal dependent set in $M$ is called a circuit of $M$, and we denote the family of all circuits of $M$ by $\mathbf{C}(M)$, i.e., $\mathbf{C}(M)=Min(2^{U}-\mathbf{I})$, where $2^{U}$ is the power set of $U$.
\end{definition}

The rank function of a matroid generalizes the maximal independence in vector subspaces.
A matroid can be defined from the viewpoint of the rank function.

\begin{definition}(Rank function \cite{Lai01Matroid})
\label{D:rank}
Let $M=(U, \mathbf{I})$ be a matroid and $X\subseteq U$.\\
\centerline{$r_{M}(X)=max\{|I|:I\subseteq X, I\in \mathbf{I}\}$,}
where $r_{M}$ is called the rank function of $M$.
\end{definition}

In order to represent the dependency between an element and a subset of a universe, we introduce the closure operator of a matroid.

\begin{definition}(Closure~\cite{Lai01Matroid})
\label{D:closure}
Let $M=(U, \mathbf{I})$ be a matroid and $X\subseteq U$.
For any $u\in U$, if $r_{M}(X)=r_{M}(X\bigcup\{u\})$, then $u$ depends on $X$. The subset of all elements depending on $X$ of $U$ is called the closure with respect to $X$ and denoted by $cl_{M}(X)$:\\
\centerline{$cl_{M}(X)=\{u\in U:r_{M}(X)=r_{M}(X\bigcup\{u\})\}$.}
\end{definition}

In the following definitions, we will introduce some special matroids used in this paper.

\begin{definition}(Free matroid~\cite{LiuChen94Matroid})
\label{D:freematroid}
Let $M=(U, \mathbf{I})$ be a matroid.
$M$ is called a free matroid if $\mathbf{I}=\{I:I\subseteq U\}$.
\end{definition}

We see that a matroid is a free matroid if any subset of its universe is an independent set.
In the following definition, we will introduce another matroid called restriction of a matroid.

\begin{definition}(Restriction~\cite{Lai01Matroid})
\label{D:restriction}
Let $M=(U, \mathbf{I})$ be a matroid. For any $X\subseteq U$, $M|X=(X, \mathbf{I}_{X})$ is called the restriction of $M$ in $X$, where $\mathbf{I}_{X}=\{I\in\mathbf{I}:I\subseteq X\}$.
\end{definition}

The following definition introduces a matroid called direct sum of matroids, which is expressed by the union of a family of matroids in different universes.

\begin{definition}(Direct sum of matroids~\cite{Lai01Matroid})
\label{D:directsumofmatroids}
Let $M_{1}=(U_{1}, \mathbf{I}_{1}), M_{2}=(U_{2}, \mathbf{I}_{2})$ be two matroids and $U_{1}\cap U_{2}=\emptyset$.
$M=(U, \mathbf{I})$ is a matroid where $U=U_{1}\cup U_{2}$ and $\mathbf{I}=\{I_{1}\cup I_{2}:I_{1}\in\mathbf{I}_{1}, I_{2}\in\mathbf{I}_{2}\}$.
We call $M$ the direct sum of $M_{1}$ and $M_{2}$, and denote it by $M=M_{1}\oplus M_{2}$.
\end{definition}

\section{A parametric matroid of rough set}
\label{S:aparametricmatroidofroughsets}
In this section, for a universe and an equivalence relation on the universe, we propose a parametric matroid of the rough set.
First, we present a parametric set family in the following definition.

\begin{definition}
\label{D:parametricsetfamily}
Let $R$ be an equivalence relation on $U$ and $X\subseteq U$.
We define a parametric set family with $X$ as its parameter as follows:\\
\centerline{$\mathbf{I}_{X}=\{I\subseteq U:\underline{R}(I)\subseteq X\}.$}
\end{definition}

In the following proposition, we will prove that the parametric set family satisfies the properties of independent sets of matroids.

\begin{proposition}
\label{P:parametricsetfamilysatisfiesindependentsetaxioms}
Let $R$ be an equivalence relation on $U$ and $X\subseteq U$.
Then, $\mathbf{I}_{X}$ satisfies (I1), (I2) and (I3) in Definition~\ref{D:matroid}.
\end{proposition}

\begin{proof}
(I1) According to (1) of Proposition~\ref{P:propertiesofapproximations}, $\underline{R}(\emptyset)=\emptyset$.
Since $\emptyset\subseteq X$, according to Definition~\ref{D:parametricsetfamily}, $\emptyset\in\mathbf{I}_{X}$.

~~~~(I2) If $I\in\mathbf{I}_{X}, I'\subseteq I$, according to Definition~\ref{D:parametricsetfamily} and (5) of Proposition~\ref{P:propertiesofapproximations}, $\underline{R}(I)\subseteq X$ and $\underline{R}(I')\subseteq\underline{R}(I)$, then $\underline{R}(I')\subseteq X$, hence $I'\in\mathbf{I}_{X}$.

~~~~(I3) If $I_{1}, I_{2}\in\mathbf{I}_{X}$ and $|I_{1}|<|I_{2}|$, then there exists $u\in I_{2}-I_{1}$ such that $I_{1}\cup\{u\}\in\mathbf{I}_{X}$.
Suppose for all $u\in I_{2}-I_{1}$, $I_{1}\cup\{u\}\notin\mathbf{I}_{X}$.
According to Definition~\ref{D:parametricsetfamily}, $\underline{R}(I_{1}\cup\{u\})\nsubseteq X$.
Since $I_{1}\in\mathbf{I}_{X}$, i.e., $\underline{R}(I_{1})\subseteq X$, therefore $P_{u}\subseteq I_{1}\cup\{u\}$ and $P_{u}\nsubseteq X$, where $u\in P_{u}\in U/R$.
For all $u_{1}, u_{2}\in I_{2}-I_{1}$, $u_{1}\neq u_{2}$, $u_{1}\in P_{u_{1}}\in U/R$ and $u_{2}\in P_{u_{2}}\in U/R$, then $P_{u_{1}}\neq P_{u_{2}}$.
Since $I_{1}=(I_{1}-I_{2})\cup(I_{1}\cap I_{2}), I_{2}=(I_{2}-I_{1})\cup(I_{1}\cap I_{2})$ and $\underline{R}(I_{2})\subseteq X$, then for all $x\in I_{2}-I_{1}$, there exists $y\in I_{1}-I_{2}$ such that $y\in P_{x}\in U/R$, therefore $|I_{2}-I_{1}|\leq |I_{1}-I_{2}|$.
Since $|I_{1}|=|I_{1}-I_{2}|+|I_{1}\cap I_{2}|, |I_{2}|=|I_{2}-I_{1}|+|I_{1}\cap I_{2}|$, then $|I_{2}|\leq |I_{1}|$, which is contradictory with the condition $|I_{1}|<|I_{2}|$.
Therefore, there exists $u\in I_{2}-I_{1}$ such that $\underline{R}(I_{1}\cup\{u\})\subseteq X$, i.e., $I_{1}\cup\{u\}\in\mathbf{I}_{X}$.
\end{proof}

From Proposition~\ref{P:parametricsetfamilysatisfiesindependentsetaxioms}, we see that the parametric set family satisfies the independent set axiom of matroids, therefore it can generate a matroid.

\begin{definition}
Let $R$ be an equivalence relation on $U$ and $X\subseteq U$.
The matroid with $\mathbf{I}_{X}$ as its independent set family is denoted by $M_{X}=(U, \mathbf{I}_{X})$.
We say $M_{X}$ is the parametric matroid of the rough set with respect to $X$.
\end{definition}

The following example illustrates the parametric matroid of the rough set.

\begin{example}
Let $R=\{(1, 1), (1, 2), (2, 1), (2, 2), (3, 3)\}$ be an equivalence relation on $U=\{1, 2, 3\}$ and $X=\{1\}$.
Then the partition induced by $R$ is $U/R=\{\{1, 2\}, \{3\}\}$.
According to Definition~\ref{D:lowerandupper}, $\underline{R}(\emptyset)=\underline{R}(\{1\})=\underline{R}(\{2\})=\emptyset, \underline{R}(\{3\})=\{3\}, \underline{R}(\{1, 2\})=\{1, 2\}, \underline{R}(\{1, 3\})=\{3\}, \underline{R}(\{2, 3\})=\{3\}, \underline{R}(\{1, 2, 3\})=\{1, 2, 3\}$, Therefore the parametric matroid with respect to $X$ is $M_{X}=(U, \mathbf{I}_{X})$, where $\mathbf{I}_{X}=\{\emptyset, \{1\}, \{2\}\}$.
\end{example}

In the following two propositions, through the lower approximation operator, we obtain two equivalent representations of the parametric set family.

\begin{proposition}
\label{P:secondformofparametricsetfamily}
Let $R$ be an equivalence relation on $U$ and $X\subseteq U$.\\
\centerline{$\mathbf{I}_{X}=\{I\subseteq U:\underline{R}(I)\subseteq\underline{R}(X)\}.$}
\end{proposition}

\begin{proof}
We need to prove $\underline{R}(I)\subseteq X\Leftrightarrow\underline{R}(I)\subseteq\underline{R}(X)$.\\
$(\Rightarrow)$: According to (5) and (7) of Proposition~\ref{P:propertiesofapproximations}, $\underline{R}(I)\subseteq X\Rightarrow\underline{R}(\underline{R}(I))\subseteq\underline{R}(X)\Rightarrow\underline{R}(I)\subseteq\underline{R}(X)$.\\
$(\Leftarrow)$: According to (3) of Proposition~\ref{P:propertiesofapproximations}, $\underline{R}(X)\subseteq X$.
Since $\underline{R}(I)\subseteq\underline{R}(X)$, then $\underline{R}(I)\subseteq X$.
\end{proof}

\begin{proposition}
\label{P:thirdformofparametricsetfamily}
Let $R$ be an equivalence relation on $U$ and $X\subseteq U$.\\
\centerline{$\mathbf{I}_{X}=\{I\subseteq U:\underline{R}(I-\underline{R}(X))=\emptyset\}.$}
\end{proposition}

\begin{proof}
According to Proposition~\ref{P:secondformofparametricsetfamily}, $\mathbf{I}_{X}=\{I\subseteq U:\underline{R}(I)\subseteq\underline{R}(X)\}$.
Therefore, we need to prove $\underline{R}(I-\underline{R}(X))=\emptyset\Leftrightarrow\underline{R}(I)\subseteq\underline{R}(X)$.
According to Proposition~\ref{P:propertiesofapproximations}, $\underline{R}(I-\underline{R}(X))=\underline{R}(I\cap(U-\underline{R}(X)))=\underline{R}(I)\cap\underline{R}(U-\underline{R}(X))=\emptyset\Leftrightarrow\underline{R}(I)\subseteq U-\underline{R}(U-\underline{R}(X))=\overline{R}(\underline{R}(X))=\underline{R}(X)$, i.e., $\underline{R}(I)\subseteq\underline{R}(X)$.
To sum up, this completes the proof.
\end{proof}

The parametric set family is based on an equivalence relation on a universe.
Since there is a one-to-one correspondence between equivalence relations and partitions, we want to know whether the parametric set family can be represented by the partition generated by the equivalence relation.

\begin{proposition}
\label{P:fourthformofparametricsetfamily}
Let $R$ be an equivalence relation on $U$ and $X\subseteq U$.\\
\centerline{$\mathbf{I}_{X}=\{I\subseteq U:\forall P\in U/R, P\nsubseteq\underline{R}(X)\Rightarrow |P\cap I|\leq |P|-1\}$.}
\end{proposition}

\begin{proof}
According to Proposition~\ref{P:thirdformofparametricsetfamily}, we need to prove $\{I\subseteq U:\underline{R}(I-\underline{R}(X))=\emptyset\}=\{I\subseteq U:\forall P\in U/R, P\nsubseteq\underline{R}(X), |P\cap I|\leq |P|-1\}$.\\
$(\Rightarrow)$: For all $I\in\{I\subseteq U:\underline{R}(I-\underline{R}(X))=\emptyset\}$, according to Definition~\ref{D:lowerandupper}, for all $P\in U/R$, $P\nsubseteq I-\underline{R}(X)$, then $|P\cap (I-\underline{R}(X))|\leq |P|-1$.
If $P\nsubseteq\underline{R}(X)$, then $|P\cap (I-\underline{R}(X))|=|(P\cap (I-\underline{R}(X)))\cup(P\cap(I\cap\underline{R}(X)))|=|P\cap((I-\underline{R}(X))\cup(I\cap\underline{R}(X)))|=|P\cap I|\leq |P|-1$.
This proves that $\{I\subseteq U:\underline{R}(I-\underline{R}(X))=\emptyset\}\subseteq\{I\subseteq U:\forall P\in U/R, P\nsubseteq\underline{R}(X), |P\cap I|\leq |P|-1\}$.\\
$(\Leftarrow)$: For all $I\in\{I\subseteq U:\forall P\in U/R, P\nsubseteq\underline{R}(X), |P\cap I|\leq |P|-1\}$, since $I\cap\underline{R}(X)\subseteq\underline{R}(X)$, then $P\cap(I\cap\underline{R}(X))=\emptyset$, therefore, $|P\cap I|=|P\cap ((I-\underline{R}(X))\cup(I\cap\underline{R}(X)))|=|(P\cap(I-\underline{R}(X)))\cup(P\cap(I\cap\underline{R}(X)))|=|P\cap(I-\underline{R}(X))|\leq |P|-1$, so $P\nsubseteq I-\underline{R}(X)$.
Since $I-\underline{R}(X)\subseteq U-\underline{R}(X)$, then for all $P\subseteq\underline{R}(X)$ where $P\in U/R$, $P\nsubseteq I-\underline{R}(X)$.
Therefore, for all $P\in U/R$, $P\nsubseteq I-\underline{R}(X)$.
According to Definition~\ref{D:lowerandupper}, $\underline{R}(I-\underline{R}(X))=\emptyset$.
This proves that $\{I\subseteq U:\underline{R}(I-\underline{R}(X))=\emptyset\}\supseteq\{I\subseteq U:\forall P\in U/R, P\nsubseteq\underline{R}(X), |P\cap I|\leq |P|-1\}$.
\end{proof}

For the parametric set family, its parameter is any subset of the universe.
We will consider the situation when the subset is equal to empty set.
First, we introduce a partition-circuit matroid induced by a partition~\cite{LiuZhu12characteristicofpartition-circuitmatroid}.
Since there is a one-to-one correspondence from a partition to an equivalence relation, a partition-circuit matroid based on an equivalence relation is introduced at follows.

\begin{definition}(Partition-circuit matroid~\cite{LiuZhu12characteristicofpartition-circuitmatroid})
\label{D:partitioncircuitmatroid}
Let $R$ be an equivalence relation on $U$.
A partition-circuit matroid $M_{R}$ is an ordered pair $(U, \mathbf{I}_{R})$ where $\mathbf{C}(M_{R})=U/R$.
\end{definition}

The independent sets of a partition-circuit matroid can be expressed by the lower approximation operator.

\begin{proposition}(\cite{LiuZhu12characteristicofpartition-circuitmatroid})
\label{P:partition-circuit'sindependent}
Let $R$ be an equivalence relation on $U$ and $M_{R}=(U, \mathbf{I}_{R})$ the partition-circuit matroid.
Then, $\mathbf{I}_{R}=\{I\subseteq U:\underline{R}(I)=\emptyset\}$.
\end{proposition}

According to Proposition~\ref{P:thirdformofparametricsetfamily} and Proposition~\ref{P:partition-circuit'sindependent}, one can see that a parametric matroid of the rough set with respect to a subset of the universe is degenerated to a partition-circuit matroid when the subset is empty set.
We will ask a question that ``what is the relationship between a parametric matroid with respect to an arbitrary subset and a partition-circuit matroid?''.
In order to answer this question, we first propose one proposition and two lemmas as follows.

\begin{proposition}
Let $R$ be an equivalence relation on $U$.
For any $X\subseteq U$, $X=\underline{R}(X)$ if and only if $R\upharpoonright X$ is an equivalence relation on $X$.
\end{proposition}

\begin{proof}
($\Rightarrow$): Since $R$ be an equivalence on $U$ and $X\subseteq U$, then for all $x\in X$, $(x, x)\in R$, i.e., $(x, x)\in R\upharpoonright X$.
Therefore, $R\upharpoonright X$ is reflexive.\\
If $(x, y)\in R\upharpoonright X$, then $(x, y)\in R$ and $x\in X$.
Since $R$ is an equivalence relation, then $(y, x)\in R$ and $y\in [x]_{R}$.
Since $X=\underline{R}(X)$, then $x\in\underline{R}(X)$.
According to Definition~\ref{D:lowerandupper}, $[x]_{R}\subseteq X$, then $y\in X$.
Therefore, $(y, x)\in R\upharpoonright X$.
Hence, $R\upharpoonright X$ is symmetric.\\
Suppose $(x, y)\in R\upharpoonright X, (y, z)\in R\upharpoonright X$, then $(x, y)\in R, (y, z)\in R, x\in X$ and $y\in X$.
Since $R$ is an equivalence relation, then $(x, z)\in R$.
Since $x\in X$, then $(x, z)\in R\upharpoonright X$.
Therefore, $R\upharpoonright X$ is transitive.\\
So $R\upharpoonright X$ is an equivalence relation on $X$.\\
($\Leftarrow$): Suppose $R\upharpoonright X$ is an equivalence relation on $X$, $X\neq\underline{R}(X)$.
According to Proposition~\ref{P:propertiesofapproximations}, $\underline{R}(X)\subseteq X$.
Therefore, there exists $x\in X-\underline{R}(X)$ such that $[x]_{R}\nsubseteq X$, then there exists $y\in [x]_{R}-X$.
Since $R$ is an equivalence relation, then $(x, y)\in R$ and $(y, x)\in R$.
Since $x\in X, y\notin X$, then $(x, y)\in R\upharpoonright X, (y, x)\notin R\upharpoonright X$ which is contradictory with the condition that $R\upharpoonright X$ is an equivalence relation.
Therefore, $X=\underline{R}(X)$.
\end{proof}

For an equivalence relation on a universe, if the lower approximation of a subset of the universe is equal to the subset itself, then the restriction of the equivalence relation in the subset is an equivalence relation on the subset.

\begin{remark}
If $R$ is an equivalence relation on $U$ and $X\subseteq U$, then $R\upharpoonright\underline{R}(X)$ is an equivalence relation on $\underline{R}(X)$ and $R\upharpoonright U-\underline{R}(X)$ is an equivalence relation on $U-\underline{R}(X)$.
Moreover $\underline{R}(X)/(R\upharpoonright\underline{R}(X))=\{P\in U/R:P\subseteq\underline{R}(X)\}$ and $(U-\underline{R}(X))/(R\upharpoonright U-\underline{R}(X))=\{P\in U/R:P\subseteq U-\underline{R}(X)\}$.
\end{remark}

\begin{lemma}
\label{L:twoequivalencerelations}
Let $R$ be an equivalence relation on $U$ and $X\subseteq U$.
For any $I\subseteq\underline{R}(X)$, $\underline{R\upharpoonright\underline{R}(X)}(I)=\emptyset$ if and only if $\underline{R}(I)=\emptyset$.
\end{lemma}

\begin{proof}
Since $I\subseteq\underline{R}(X)$, according to Definition~\ref{D:lowerandupper}, $\underline{R\upharpoonright\underline{R}(X)}(I)=\emptyset\Leftrightarrow\cup\{P\in U/R:P\subseteq\underline{R}(X), P\subseteq I\}=\emptyset\Leftrightarrow\cup\{P\in U/R:P\subseteq I\}=\emptyset\Leftrightarrow\underline{R}(I)=\emptyset$.
\end{proof}

\begin{lemma}
\label{L:twoequallowerapproximations}
Let $R$ be an equivalence relation on $U$, $X\subseteq U$ and $X_{1}\subseteq\underline{R}(X), X_{2}\subseteq U-\underline{R}(X)$.
If $\underline{R}(X_{1})=\emptyset$, then $\underline{R}(X_{1}\cup X_{2})=\underline{R}(X_{2})$.
\end{lemma}

\begin{proof}
According to (5) of Proposition~\ref{P:propertiesofapproximations}, $X_{2}\subseteq X_{1}\cup X_{2}$, then $\underline{R}(X_{2})\subseteq\underline{R}(X_{1}\cup X_{2})$.
When $\underline{R}(X_{1})=\emptyset$, suppose that $\underline{R}(X_{1}\cup X_{2})\neq\underline{R}(X_{2})$, then there exists $x$ such that $x\in\underline{R}(X_{1}\cup X_{2})-\underline{R}(X_{2})$.
Therefore, there exists $x\in P_{x}\in U/R$ such that $P_{x}\subseteq\underline{R}(X_{1}\cup X_{2})-\underline{R}(X_{2})\subseteq\underline{R}(X_{1}\cup X_{2})$.
According to Definition~\ref{D:lowerandupper}, $P_{x}\subseteq X_{1}\cup X_{2}$.
Since $X_{1}\subseteq\underline{R}(X), X_{2}\subseteq U-\underline{R}(X)$, then $P_{x}\subseteq X_{1}$ or $P_{x}\subseteq X_{2}$.
Since $P_{x}\subseteq\underline{R}(X_{1}\cup X_{2})-\underline{R}(X_{2})$, then $P_{x}\nsubseteq X_{2}$, hence $P_{x}\subseteq X_{1}$, which is contradictory with $\underline{R}(X_{1})=\emptyset$.
Therefore, $\underline{R}(X_{1}\cup X_{2})=\underline{R}(X_{2})$ if $\underline{R}(X_{1})=\emptyset$.
\end{proof}

In the following theorem, the parametric matroid of the rough set is proved to be the direct sum of a partition-circuit matroid and a free matroid.

\begin{theorem}
\label{T:asubsetfamilysatifiesindependent}
Let $R$ be an equivalence relation on $U$ and $X\subseteq U$.
Let $M_{R\upharpoonright U-\underline{R}(X)}=(U-\underline{R}(X), \mathbf{I}_{R\upharpoonright U-\underline{R}(X)})$ be a partition-circuit matroid and $M=(\underline{R}(X), \mathbf{I})$ a free matroid.
Then,\\
\centerline{$\mathbf{I}_{X}=\{I_{1}\cup I_{2}:I_{1}\in\mathbf{I}_{R\upharpoonright U-\underline{R}(X)}, I_{2}\in\mathbf{I}\}$.}
\end{theorem}

\begin{proof}
According to Definition~\ref{D:freematroid} and Proposition~\ref{P:partition-circuit'sindependent}, $\mathbf{I}_{R\upharpoonright U-\underline{R}(X)}=\{I\subseteq U-\underline{R}(X):\underline{R\upharpoonright U-\underline{R}(X)}(I)=\emptyset\}$ and $\mathbf{I}=\{I:I\subseteq\underline{R}(X)\}$.
According to Definition~\ref{D:parametricsetfamily}, we only need to prove that $\{I\subseteq U:\underline{R}(I)\subseteq X\}=\{I_{1}\cup I_{2}:I_{1}\in\mathbf{I}_{R\upharpoonright U-\underline{R}(X)}, I_{2}\in\mathbf{I}\}$.\\
($\Rightarrow$): For all $I\in\mathbf{I}_{X}$, since $I=(I-\underline{R}(X))\cup(I\cap\underline{R}(X))$, then $(I-\underline{R}(X))\cup(I\cap\underline{R}(X))\in\mathbf{I}_{X}$, hence $\underline{R}((I-\underline{R}(X))\cup(I\cap\underline{R}(X)))\subseteq X$.
According to (6) of Proposition~\ref{P:propertiesofapproximations}, $\underline{R}(I-\underline{R}(X))\cup\underline{R}(I\cap\underline{R}(X))\subseteq \underline{R}((I-\underline{R}(X))\cup(I\cap\underline{R}(X)))$, then $\underline{R}(I-\underline{R}(X))\subseteq X$.
According to (3) of Proposition~\ref{P:propertiesofapproximations}, $\underline{R}(X)\subseteq X$.
Therefore, $\underline{R}(I-\underline{R}(X))\cup\underline{R}(X)\subseteq X$.
According to (5), (6) and (7) of Proposition~\ref{P:propertiesofapproximations}, $\underline{R}(\underline{R}(I-\underline{R}(X))\cup\underline{R}(X))\subseteq\underline{R}(X)$, then $\underline{R}(I-\underline{R}(X))\cup\underline{R}(X)\subseteq\underline{R}(X)$.
Therefore, $\underline{R}(I-\underline{R}(X))=\emptyset$.
Since $I-\underline{R}(X)\subseteq U-\underline{R}(X)$, according to Lemma~\ref{L:twoequivalencerelations}, $\underline{R\upharpoonright U-\underline{R}(X)}(I-\underline{R}(X))=\emptyset$, then $I-\underline{R}(X)\in\mathbf{I}_{R\upharpoonright U-\underline{R}(X)}$.
Since $I\cap\underline{R}(X)\subseteq\underline{R}(X)$, then $I\cap\underline{R}(X)\in\mathbf{I}$, therefore $(I-\underline{R}(X))\cup(I\cap\underline{R}(X))\in\{I_{1}\cup I_{2}:I_{1}\in\mathbf{I}_{R\upharpoonright U-\underline{R}(X)}, I_{2}\in\mathbf{I}\}$, i.e., $I\in\{I_{1}\cup I_{2}:I_{1}\in\mathbf{I}_{R\upharpoonright U-\underline{R}(X)}, I_{2}\in\mathbf{I}\}$.\\
($\Leftarrow$): For all $I\in\{I_{1}\cup I_{2}:I_{1}\in\mathbf{I}_{R\upharpoonright U-\underline{R}(X)}, I_{2}\in\mathbf{I}\}$, there exist $I_{1}\in\mathbf{I}_{R\upharpoonright U-\underline{R}(X)}$ and $I_{2}\in\mathbf{I}$ such that $I=I_{1}\cup I_{2}$.
Since $\underline{R\upharpoonright U-\underline{R}(X)}(I_{1})=\emptyset$, according to Lemma~\ref{L:twoequivalencerelations}, we obtain $\underline{R}(I_{1})=\emptyset$.
Since $I_{2}\subseteq\underline{R}(X)$, according to (5) and (7) of Proposition~\ref{P:propertiesofapproximations}, we obtain $\underline{R}(I_{2})\subseteq\underline{R}(X)$.
According to Lemma~\ref{L:twoequallowerapproximations}, we obtain that $\underline{R}(I_{1}\cup I_{2})=\underline{R}(I_{2})$, i.e., $\underline{R}(I)=\underline{R}(I_{2})$.
Therefore, $\underline{R}(I)\subseteq\underline{R}(X)$.
According to (3) of Proposition~\ref{P:propertiesofapproximations}, $\underline{R}(I)\subseteq X$.
So, $I\in\mathbf{I}_{X}$.
\end{proof}

For a universe and an equivalence relation on the universe, the parametric matroid of the rough set with respect to a subset of the universe is the direct sum of a partition-circuit matroid and a free matroid, where the partition-circuit matroid is based on the restriction of the equivalence relation in the complement of the lower approximation of the subset and the free matroid is based on the lower approximation of the subset.
Moreover, can the partition-circuit matroid be expressed by the parametric matroid of the rough set? And what about the free matroid?
In the following proposition, we will solve these issues.

\begin{proposition}
\label{P:restrictionofdirectsum}
Let $M_{1}=(U_{1}, \mathbf{I}_{1}), M_{2}=(U_{2}, \mathbf{I}_{2})$ be two matroids and $M=(U, \mathbf{I})$ the direct sum of $M_{1}$ and $M_{2}$.
Then $M_{1}=M|U_{1}, M_{2}=M|U_{2}$.
\end{proposition}

\begin{proof}
According to Definition~\ref{D:directsumofmatroids} and Definition~\ref{D:restriction}, it is straightforward.
\end{proof}

For the direct sum of matroids, any one of the matroids is the restriction of the direct sum.
Therefore, a partition-circuit matroid on a universe is the restriction of the parametric matroid in the universe, and the same as a free matroid.

\section{Characteristics of a parametric matroid through the lower approximation number}
\label{S:characteristicsoftheparametricmatroid}
As shown in Section~\ref{S:aparametricmatroidofroughsets}, a parametric set family determines a parametric matroid, and vice versa.
Moreover, a parametric matroid is the direct sum of a partition-circuit matroid and a free matroid.
Through a tool called the lower approximation number, some characteristics of partition-circuit matroids can be well represented.
Can the lower approximation number be applied to a parametric matroid?
First, in the following definition, we will introduce the lower approximation number.

\begin{definition}(Lower approximation number~\cite{LiuZhu12characteristicofpartition-circuitmatroid})
\label{D:lowerapproximationnumber}
Let $R$ be an equivalence relation on $U$ and $X\subseteq U$.
We define the lower approximation number of $X$ with respect to $R$ as follows:\\
\centerline{$f_{R}(X)=|\{P\in U/R:P\subseteq X\}|$.}
\end{definition}

One can see that the lower approximation number of any subset of a universe is equal to the number of equivalence classes which the subset contains.
The following proposition represents the parametric set family through the lower approximation number.

\begin{proposition}
\label{P:fifthformofparametricsetfamily}
Let $R$ be an equivalence relation on $U$ and $X\subseteq U$.\\
\centerline{$\mathbf{I}_{X}=\{I\subseteq U:f_{R}(I-\underline{R}(X))=0\}$.}
\end{proposition}

\begin{proof}
According to Proposition~\ref{P:thirdformofparametricsetfamily} and Definition~\ref{D:lowerapproximationnumber}, it is straightforward.
\end{proof}

Base is one of important characteristics of matroids.
We will investigate it of the parametric matroid of the rough set through the lower approximation number as follows.

\begin{proposition}
Let $R$ be an equivalence relation on $U$, $X\subseteq U$ and $M_{X}=(U, \mathbf{I}_{X})$ the parametric matroid of the rough set with respect to $X$.
Then,\\
\centerline{$\mathbf{B}(M_{X})=Max\{I\subseteq U:f_{R}(I-\underline{R}(X))=0\}$.}
\end{proposition}

\begin{proof}
According to Definition~\ref{D:base} and Proposition~\ref{P:fifthformofparametricsetfamily}, it is straightforward.
\end{proof}

The following proposition represents the base set family of the parametric matroid of the rough set through a partition.

\begin{proposition}
Let $R$ be an equivalence relation on $U$, $X\subseteq U$ and $M_{X}=(U, \mathbf{I}_{X})$ the parametric matroid of the rough set with respect to $X$.
Then,\\
\centerline{$\mathbf{B}(M_{X})=\{I\cup\underline{R}(X)\subseteq U:\forall P\in U/R, P\nsubseteq\underline{R}(X)\Rightarrow |P\cap I|=|P|-1\}$.}
\end{proposition}

\begin{proof}
According to Proposition~\ref{P:fourthformofparametricsetfamily}, it is straightforward.
\end{proof}

Through the lower approximation number, the circuits of the parametric matroid of the rough set are represented in the following proposition.

\begin{proposition}
Let $R$ be an equivalence relation on $U$, $X\subseteq U$ and $M_{X}=(U, \mathbf{I}_{X})$ a parametric matroid of the rough set with respect to $X$.
Then,\\
\centerline{$\mathbf{C}(M_{X})=Min\{C\subseteq U:f_{R}(C-\underline{R}(X))=1\}$.}
\end{proposition}

\begin{proof}
According to Definition~\ref{D:circuit} and Proposition~\ref{P:fifthformofparametricsetfamily}, it is straightforward.
\end{proof}

A parametric matroid is the direct sum of a partition-circuit matroid and a free matroid.
We will investigate the circuits of the parametric matroid through the circuits of the partition-circuit matroid and the free matroid.
First, we introduce a proposition~\cite{LiuZhuZhang12Relationshipbetween} which shows the relationship between the circuits of a matroid and ones of its restrictions.

\begin{proposition}(\cite{LiuZhuZhang12Relationshipbetween})
\label{P:circuitofdirectsum}
Let $M=(U, \mathbf{I})$ be a matroid, $U=U_{1}\cup U_{2}$ and $U_{1}\cap U_{2}=\emptyset$.
Then,\\
\centerline{$\mathbf{C}(M)=\mathbf{C}(M|U_{1})\cup\mathbf{C}(M|U_{2})$.}
\end{proposition}

The circuits of a parametric matroid are obtained in the following proposition.

\begin{proposition}
Let $R$ be an equivalence relation on $U$, $X\subseteq U$ and $M_{X}=(U, \mathbf{I}_{X})$ the parametric matroid of the rough set with respect to $X$.
Then,\\
\centerline{$\mathbf{C}(M_{X})=\{P\in U/R:P\subseteq U-\underline{R}(X)\}$.}
\end{proposition}

\begin{proof}
According to Theorem~\ref{T:asubsetfamilysatifiesindependent}, $M_{X}=M_{R\upharpoonright U-\underline{R}(X)}\oplus M$, where $M_{R\upharpoonright U-\underline{R}(X)}$ is the partition-circuit matroid and $M$ is a free matroid.
According to Definition~\ref{D:partitioncircuitmatroid} and Definition~\ref{D:freematroid}, $\mathbf{C}(M_{R\upharpoonright U-\underline{R}(X)})=(U-\underline{R}(X))/(R\upharpoonright U-\underline{R}(X))$ and $\mathbf{C}(M)=\emptyset$.
According to Proposition~\ref{P:restrictionofdirectsum} and Proposition~\ref{P:circuitofdirectsum}, $\mathbf{C}(M_{X})=\mathbf{C}(M_{R\upharpoonright U-\underline{R}(X)})\cup\mathbf{C}(M)$.
Hence $\mathbf{C}(M_{X})=\{P\in U/R:P\subseteq U-\underline{R}(X)\}$.
\end{proof}

The rank function is a quantitative tool of matroids.
In the following, we will study the rank function of a parametric matroid.
We first investigate the relationship between the rank function of the direct sum of two matroids and the rank functions of the two matroids.

\begin{proposition}
\label{P:relationshipbetweenrankofdirectsum}
Let $M_{1}=(U_{1}, \mathbf{I}_{1}), M_{2}=(U_{2}, \mathbf{I}_{2})$ be two matroids and $M=(U, \mathbf{I})$ the direct sum of $M_{1}$ and $M_{2}$.
Then for all $X\subseteq U$,\\
\centerline{$r_{M}(X)=r_{M_{1}}(X\cap U_{1})+r_{M_{2}}(X\cap U_{2})$.}
\end{proposition}

\begin{proof}
$(\Rightarrow)$: Suppose $r_{M}(X)=|B|$, according to Definition~\ref{D:rank}, $B\subseteq X$ and $B\in\mathbf{I}$.
$B=B\cap U=B\cap (U_{1}\cup U_{2})=(B\cap U_{1})\cup (B\cap U_{2})$, suppose $B_{1}=B\cap U_{1}, B_{2}=B\cap U_{2}$, then $B=B_{1}\cup B_{2}$.
According to (I2) of Definition~\ref{D:matroid}, $B_{1}\in\mathbf{I}, B_{2}\in\mathbf{I}$.
Since $B_{1}\subseteq U_{1}, B_{2}\subseteq U_{2}$, according to Definition~\ref{D:restriction} and Proposition~\ref{P:restrictionofdirectsum}, $B_{1}\in\mathbf{I}_{1}, B_{2}\in\mathbf{I}_{2}$.
Since $B_{1}\subseteq X\cap U_{1}, B_{2}\subseteq X\cap U_{2}$, then $|B_{1}|\leq r_{M_{1}}(X\cap U_{1}), |B_{2}|\leq r_{M_{2}}(X\cap U_{2})$, therefore $|B|=|B_{1}|+|B_{2}|\leq r_{M_{1}}(X\cap U_{1})+r_{M_{2}}(X\cap U_{2})$, i.e., $r_{M}(X)\leq r_{M_{1}}(X\cap U_{1})+r_{M_{2}}(X\cap U_{2})$.\\
$(\Leftarrow)$: Suppose $r_{M_{1}}(X\cap U_{1})=|B_{1}|, r_{M_{2}}(X\cap U_{2})=|B_{2}|$, according to Definition~\ref{D:rank}, $B_{1}\subseteq X\cap U_{1}, B_{1}\in\mathbf{I}_{1}$ and $B_{2}\subseteq X\cap U_{2}, B_{2}\in\mathbf{I}_{2}$.
According to Definition~\ref{D:directsumofmatroids}, $B_{1}\cup B_{2}\in\mathbf{I}$.
Since $B_{1}\subseteq X\cap U_{1}, B_{2}\subseteq X\cap U_{2}$, then $B_{1}\cup B_{2}\subseteq (X\cap U_{1})\cup(X\cap U_{2})=X\cap (U_{1}\cup U_{2})=X\cap U$, i.e., $B_{1}\cup B_{2}\subseteq X$.
According to Definition~\ref{D:rank}, $r_{M}(X)\geq |B_{1}\cup B_{2}|=|B_{1}|+|B_{2}|$, i.e., $r_{M}(X)\geq r_{M_{1}}(X\cap U_{1})+r_{M_{2}}(X\cap U_{2})$.\\
To sum up, this completes the proof.
\end{proof}

A parametric matroid is the direct sum of a partition-circuit matroid and a free matroid.
In the following proposition, we will introduce the rank function of a partition-circuit matroid.

\begin{proposition}(\cite{LiuZhu12characteristicofpartition-circuitmatroid})
\label{P:rankofpartitioncircuitmatroid}
Let $R$ be an equivalence relation on $U$ and $M_{R}=(U, \mathbf{I}_{R})$ the partition-circuit matroid.
Then for all $X\subseteq U$, $r_{M_{R}}(X)=|X|-f_{R}(X)$.
\end{proposition}

The rank function of the parametric matroid of the rough set is investigated through the lower approximation number in the following proposition.

\begin{proposition}
\label{P:rankofaparametricmatroid}
Let $R$ be an equivalence relation on $U$, $X\subseteq U$ and $M_{X}=(U, \mathbf{I}_{X})$ the parametric matroid of the rough set with respect to $X$.
Then for all $Y\subseteq U$,\\
\centerline{$r_{M_{X}}(Y)=|Y|-f_{R}(Y-\underline{R}(X))$.}
\end{proposition}

\begin{proof}
According to Theorem~\ref{T:asubsetfamilysatifiesindependent}, $M_{X}=M_{R\upharpoonright U-\underline{R}(X)}\oplus M$, where $M_{R\upharpoonright U-\underline{R}(X)}=(U-\underline{R}(X), \mathbf{I}_{R\upharpoonright U-\underline{R}(X)})$ is a partition-circuit matroid and $M=(\underline{R}(X), \mathbf{I})$ is a free matroid.
According to Proposition~\ref{P:relationshipbetweenrankofdirectsum} and Proposition~\ref{P:rankofpartitioncircuitmatroid}, $r_{M_{X}}(Y)=r_{M_{R\upharpoonright U-\underline{R}(X)}}(Y\cap (U-\underline{R}(X)))+r_{M}(Y\cap\underline{R}(X))=r_{M_{R\upharpoonright U-\underline{R}(X)}}(Y-\underline{R}(X))+r_{M}(Y\cap\underline{R}(X))=|Y-\underline{R}(X)|-f_{R\upharpoonright U-\underline{R}(X)}(Y-\underline{R}(X))+|Y\cap\underline{R}(X)|=|Y|-f_{R\upharpoonright U-\underline{R}(X)}(Y-\underline{R}(X))$.
Since $Y-\underline{R}(X)\subseteq U-\underline{R}(X)$, according to Definition~\ref{D:lowerandupper} and Definition~\ref{D:lowerapproximationnumber}, then $f_{R\upharpoonright U-\underline{R}(X)}(Y-\underline{R}(X))=f_{R}(Y-\underline{R}(X))$, therefore, $r_{M_{X}}(Y)=|Y|-f_{R}(Y-\underline{R}(X))$.
\end{proof}

In a matroid, the closure of a subset is all those elements when added to the subset, the rank is the same.
The rank function of a parametric matroid can be expressed by the lower approximation number.
Moreover, we use the lower approximation number to study the closure operator of a parametric matroid.

\begin{proposition}
Let $R$ be an equivalence relation on $U$, $X\subseteq U$ and $M_{X}=(U, \mathbf{I}_{X})$ the parametric matroid of the rough set with respect to $X$.
Then for all $Y\subseteq U$,\\
\centerline{$cl_{M_{X}}(Y)=Y\cup\{y\in U-Y:f_{R}(Y\cup\{y\}-\underline{R}(X))- f_{R}(Y-\underline{R}(X))=1\}$.}
\end{proposition}

\begin{proof}
According to Definition~\ref{D:closure}, $cl_{M_{X}}(Y)=\{y\in U:r_{M_{X}}(Y)=r_{M_{X}}(Y\cup y)\}$.
If $y\in Y$, then $r_{M_{X}}(Y)=r_{M_{X}}(Y\cup y)$, hence $cl_{M_{X}}(Y)=Y\cup\{y\in U-Y:r_{M_{X}}(Y)=r_{M_{X}}(Y\cup y)\}$.
According to Proposition~\ref{P:rankofaparametricmatroid}, $r_{M_{X}}(Y)=|Y|-f_{R}(Y-\underline{R}(X)), r_{M_{X}}(Y\cup\{y\})=|Y\cup\{y\}|-f_{R}(Y\cup\{y\}-\underline{R}(X))$, if $y\notin Y$, then $r_{M_{X}}(Y)=r_{M_{X}}(Y\cup y)\Leftrightarrow |Y|-f_{R}(Y-\underline{R}(X))=|Y|+1-f_{R}(Y\cup\{y\}-\underline{R}(X))\Leftrightarrow f_{R}(Y\cup\{y\}-\underline{R}(X))-f_{R}(Y-\underline{R}(X))=1$.
To sum up, this completes the proof.
\end{proof}


\section{Conclusions}
\label{S:conclusions}
In this paper, for a universe and an equivalence relation on the universe, we proposed a parametric matroid of the rough set through defining a parametric set family based on the lower approximation operator.
Some equivalent forms of the parametric set family were obtained.
Moreover, we proved the parametric matroid of the rough set to be the direct sum of a partition-circuit matroid and a free matroid.
Through the lower approximation number, some characteristics of the parametric matroid of the rough set, such as independent sets, bases, circuits, the rank function and the closure operator, were well represented.
In future works, we will extend equivalence relations/partitions to arbitrary relations/coverings to connect matroids with generalized rough sets.

\section*{Acknowledgments}
This work is supported in part by the National Natural Science Foundation of China under Grant No. 61170128, the Natural Science Foundation of Fujian Province, China, under Grant Nos. 2011J01374 and 2012J01294, the Science and Technology Key Project of Fujian Province, China, under Grant No. 2012H0043 and State key laboratory of management and control for complex systems open project under Grant No. 20110106.


\end{document}